\documentclass[11pt]{article}

\usepackage{graphicx} 

\usepackage{natbib}

\usepackage{algorithm}
\usepackage{algorithmic}

\usepackage{hyperref}

\usepackage{enumerate}
\usepackage{bm}
\usepackage{amssymb}
\usepackage{amsthm}
\usepackage{amsmath}
\usepackage{graphicx}
\usepackage{mathtools}
\usepackage{footnote}
\usepackage{fullpage}
\allowdisplaybreaks

\newcommand{\1}{{\bm 1}}
\newcommand{\x}{{\bf x}}
\newcommand{\w}{{\bf w}}

\newcommand{\z}{{\bf z}}
\newcommand{\R}{\mathbb{R}}
\newcommand{\E}{\mathbb{E}}
\newcommand{\Var}{\text{Var}}
\newcommand{\sign}{{\text {sign}}}
\newcommand{\WL}{{\text {WL}}}

\newcommand{\specialcell}[2][c]{\begin{tabular}[#1]{@{}c@{}}#2\end{tabular}}

\renewcommand{\H}{\mathcal{H}}
\newcommand{\X}{\mathcal{X}}
\newcommand{\tS}{\tilde{S}}
\newcommand{\half}{\tfrac{1}{2}}
\newcommand{\halfgamma}{\tfrac{\gamma}{2}}
\newcommand{\deltainv}{\tfrac{1}{\delta}}
\newcommand{\hide}[1]{}

\newtheorem{theorem}{Theorem}

\newtheorem{lemma}{Lemma}

\makesavenoteenv{algorithm} 
\makesavenoteenv{algorithmic} 

\title{Optimal and Adaptive Algorithms for Online Boosting}
\author{Alina Beygelzimer \\ Yahoo Labs \\ New York, NY 10036 \\ beygel@yahoo-inc.com
\and
Satyen Kale \\ Yahoo Labs \\ New York, NY 10036 \\ satyen@yahoo-inc.com
\and
Haipeng Luo \\ Princeton University \\ Princeton, NJ 08540 \\ haipengl@cs.princeton.edu}
\begin{document} 
\maketitle

\begin{abstract} 
We study online boosting, the task of converting any weak online learner into a strong online learner. Based on a novel and natural definition of weak online learnability, we develop two online boosting algorithms. The first algorithm is an online version of boost-by-majority. By proving a matching lower bound, we show that this algorithm is essentially optimal in terms of the number of weak learners and the sample complexity needed to achieve a specified accuracy. This optimal algorithm is not adaptive, however. Using tools from online loss minimization, we derive an adaptive online boosting algorithm that is also parameter-free, but not optimal. Both algorithms work with base learners that can handle example importance weights directly, as well as by rejection sampling examples with probability defined by the booster. Results are complemented with an experimental study.
\end{abstract} 

\section{Introduction}
We study online boosting, the task of boosting the accuracy of 
any weak online learning algorithm.
The theory of boosting in the batch setting has been studied extensively
in the literature and has led to a huge practical success.
See the book by \citet{SchapireFr12} for a thorough discussion.

Online learning algorithms receive examples one by one, 
updating the predictor immediately after seeing each new example. 
In contrast to the batch setting, online learning algorithms typically 
don't make any stochastic assumptions about the data they observe.
They are often much faster, more memory-efficient, and apply to situations
where the best predictor changes over time as new examples keep coming in.

Given the success of boosting in batch learning, it is natural to ask about the possibility of applying boosting to online learning. 
Indeed, there has already been some work on online 
boosting~\citep{OzaRu01, GrabnerBi06, LiuYu07, GrabnerLeBi08, ChenLiLu12, 
ChenLiLu14}.

From a theoretical viewpoint, recent work by \citet{ChenLiLu12} is perhaps most interesting. They generalized the batch weak learning assumption to the online setting, and made a connection between online boosting and batch boosting that produces smooth distributions over the training examples.  The resulting algorithm is guaranteed to achieve an arbitrarily small error rate as long as the number of weak learners and the number of examples are sufficiently large.  No assumptions need to be made about how the data is generated. Indeed, the data can even be generated by an adversary.

We present a new online boosting algorithm, based on the boost-by-majority (BBM) algorithm of \citep{Freund95}. This algorithm, called Online BBM, improves upon the work of \citet{ChenLiLu12} in several different aspects: 
\vspace*{-0.2cm}
\begin{enumerate}
\itemsep1pt \parskip0pt \parsep0pt
	\item our assumption on online weak learners is weaker and can be seen as a direct online analogue of the weak learning assumption in standard batch boosting,
	\item our algorithm doesn't require weighted online learning, instead using a sampling technique similar to the one used in boosting by filtering in the batch setting \citep[see for example,][]{Freund92, BradleySc08}, and
	\item our algorithm is optimal in the sense that no online boosting algorithm can achieve the same error rate with less weak learners or less examples asymptotically (see the lower bounds in Section~\ref{subsec:lower_bounds}). 
\end{enumerate}
A quantitative comparison of our results with those of \citet{ChenLiLu12} appears in Table \ref{tab:results}, where $N$ and $T$ represent the number\footnote{In this paper, we use the $\tilde{O}(\cdot)$ and $\tilde{\Omega}(\cdot)$ notation to suppress dependence on polylogarithmic factors in the natural parameters.} of weak learners and examples needed to achieve error rate $\epsilon$, and $\gamma$ stands for a similar concept of the ``edge'' of the weak learning oracle as in the batch setting (smaller $\gamma$ means more inaccurate weak learners). 

A clear drawback of all the algorithms mentioned above is lack of adaptivity. A simple interpretation of this drawback is that all these algorithms require using $\gamma$, an unknown quantity, as a parameter.
More importantly, this also means that the algorithm treats each weak learner equally and ignores the fact that some weak learners are actually doing better than the others. The best example of adaptive boosting algorithm is the well-known parameter-free AdaBoost algorithm \citep{FreundSc97}, where each weak learner is naturally weighted by how accurate it is. In fact, adaptivity is known to be one of the key features that lead to the practical success of AdaBoost, and therefore should also be essential to the performance of online boosting algorithms. In Section \ref{sec:AdaBoost.OL}, we thus propose AdaBoost.OL, an adaptive and parameter-free online boosting algorithm. As shown in Table \ref{tab:results},  AdaBoost.OL is theoretically suboptimal in terms of $N$ and $T$. However, empirically it generally outperforms OSBoost and sometimes even beats the optimal algorithm Online BBM (see Section \ref{sec:experiments}).

\begin{table}[t]
\caption{Comparisons of our results with those of \citet{ChenLiLu12}, assuming, as in their paper, that the weak learner is derived from an online learning algorithm with an $O(\sqrt{T})$ regret bound.}
\label{tab:results}
\begin{center}
\begin{tabular}{|c|c|c|c|c|}
\hline
Algorithm & N & T & Optimal? & Adaptive? \\

\hline
\specialcell{Online BBM \\ {\scriptsize(Section \ref{subsec:BBM})}} & 
$O(\frac{1}{\gamma^2}\ln\frac{1}{\epsilon})$ &
$\tilde{O}(\frac{1}{\epsilon\gamma^2})$ & $\surd$ & $\times$
\\

\hline
\specialcell{AdaBoost.OL \\ {\scriptsize(Section \ref{sec:AdaBoost.OL})}}& 
$O(\frac{1}{\epsilon\gamma^2})$ & 
$\tilde{O}(\frac{1}{\epsilon^2\gamma^4})$ & $\times$ & $\surd$
\\

\hline
\specialcell{OSBoost \\ {\scriptsize\citep{ChenLiLu12}}}& 
$O(\frac{1}{\epsilon\gamma^2})$ & 
$\tilde{O}(\frac{1}{\epsilon\gamma^2})$ & $\times$ & $\times$
\\
\hline
\end{tabular}
\end{center}
\end{table}

Our techniques are also very different from those of \citet{ChenLiLu12}, which rely on the smooth boosting algorithm of \citet{Servedio03}. As far as we know, all other work on smooth boosting \citep{BshoutyGa03, BradleySc08, BarakHaKa09} cannot be easily generalized to the online setting, necessitating completely different methods not relying on smooth distributions. 	Our Online BBM algorithm builds on top of a potential based family that arises naturally in the batch setting as approximate minimax optimal algorithms for so-called drifting games \citep{Schapire01, LuoSc14b}. The decomposition of each example in that framework naturally allows us to generalize it to the online setting where example comes one by one. On the other hand, AdaBoost.OL is derived by viewing boosting from a different angle: loss minimization \citep{MasonBaBaFr99b, SchapireFr12}.
The theory of online loss minimization is the key tool for developing AdaBoost.OL.

Finally, in Section \ref{sec:experiments}, experiments on benchmark data are conducted to show that our new algorithms indeed improve over previous work.

\section{Setup and Assumptions}\label{sec:setup}
We describe the formal setup of the task of online classification by boosting.
At each time step $t=1,\ldots, T$, an adversary chooses an example $(\x_t, y_t) \in \X \times \{-1,1\}$, where $\X$ is the domain, and reveals $\x_t$ to the online learner. The learner makes a prediction on its label $\hat{y}_t \in \{-1,1\}$, and suffers the 0-1 loss $\1\{\hat{y}_t \neq y_t\}$. As is usual with online algorithms, this prediction may be randomized.

For parameters $\gamma \in (0, \half)$, $\delta \in (0, 1)$, and a constant $S > 0$, the learner is said to be a {\em weak} online learner with edge $\gamma$ and {\em excess loss} $S$ if, for any $T$ and for any input sequence of examples $(\x_t, y_t)$ for $t = 1, 2, \ldots, T$ chosen adaptively, it generates predictions $\hat{y}_t$ such that with probability at least $1 - \delta$,
\begin{equation} \label{eq:wl-guarantee}
	\sum_{t=1}^T \1\{\hat{y}_t \neq y_t\} \leq (\half - \gamma)T + S. 
\end{equation}
The excess loss requirement is necessary since an online learner can't be expected to predict with any accuracy with too few examples. Essentially, the excess loss $S$ yields a kind of sample complexity bound: the weak learner starts obtaining a distinct edge of $\Omega(\gamma)$ over random guessing when $T \gg \frac{S}{\gamma}$. Typically, the dependence of the high probability bound on $\delta$ is polylogarithmic in $\deltainv$; thus in the following we will avoid explicitly mentioning $\delta$.

For a given parameter $\epsilon > 0$, the learner is said to be a {\em strong} online learner with error rate $\epsilon$ if it satisfies the same conditions as a weak online learner except that its edge is $\half - \epsilon$, or in other words, the fraction of mistakes made, asymptotically, is $\epsilon$. Just as for the weak learner, the excess loss $S$ yields a sample complexity bound: the fraction of mistakes made by the strong learner becomes $O(\epsilon)$ when $T~\gg~\frac{S}{\epsilon}$. 

Our main theorem is the following:
\begin{theorem} \label{thm:main}
	Given a weak online learning algorithm with edge $\gamma$ and excess loss $S$ and any target error rate $\epsilon > 0$, there is a strong online learning algorithm with error rate $\epsilon$ which uses $O(\frac{1}{\gamma^2}\ln(\frac{1}{\epsilon}))$ copies of the weak online learner, and has excess loss $\tilde{O}(\frac{S}{\gamma} + \frac{1}{\gamma^2})$; thus its sample complexity is $\tilde{O}(\frac{1}{\epsilon}(\frac{S}{\gamma} + \frac{1}{\gamma^2}))$.  Furthermore, if $S \geq \tilde{\Omega}(\tfrac{1}{\gamma})$, then the number of weak online learners is optimal up to constant factors, and the sample complexity is optimal up to polylogarithmic factors.
\end{theorem}

The requirement that $S \geq \tilde{\Omega}(\tfrac{1}{\gamma})$ in the lower bound is not very stringent; this is precisely the excess loss one obtains when using standard online learning algorithms with regret bound $O(\sqrt{T})$, as is explained in the discussion following Lemma~\ref{lem:agnostic-to-weak}. Furthermore, since we require the bound (\ref{eq:wl-guarantee}) to hold with high probability, typical analyses of online learning algorithms will have an $\tilde{O}(\sqrt{T})$ deviation term, which also leads to $S \geq \tilde{\Omega}(\tfrac{1}{\gamma})$.

As the theorem indicates, the strong online learner (hereafter referred to as ``booster'') works by maintaining $N$ copies of the weak online learner, for some positive integer $N$ to be specified later. Denote the weak online learners $\WL^i$ for $i = 1, 2, \ldots, N$. At time step $t$, the prediction of $i$-th weak online learner is given by $\WL^i(\x_t) \in \{-1, 1\}$. Note the slight abuse of notation here: $\WL^i$ is {\em not} a function, rather it is an algorithm with an internal state that is updated as it is fed training examples. Thus, the prediction $\WL^i(\x_t)$ depends on the internal state of $\WL^i$, and for notational convenience we avoid reference to the internal state.

In each round $t$, the booster works by taking a weighted majority vote of the weak learners' predictions. Specifically, the booster maintains weights $\alpha_t^i \in \R$ for $i = 1, \ldots, N$ corresponding to each weak learner, and its final prediction will then be\footnote{In Section \ref{sec:AdaBoost.OL} a slightly different final prediction will be used.
} 
$\hat{y}_t = \sign(\sum_{i=1}^N \alpha_t^i \WL^i(\x_t))$,
where $\sign(\cdot)$ is $1$ if the argument is nonnegative and $-1$ otherwise.
After making the prediction, the true label $y_t$ is revealed by the environment. The booster then updates $\WL^i$ by passing the training example $(\x_t, y_t)$ to $\WL^i$ with a carefully chosen sampling probability $p_t^i$ (and not passing the example with the remaining probability). The sampling probability $p_t^i$ is obtained by computing a weight $w_t^i$ and setting\footnote{In the algorithm we simply use a tight-enough upper bound on $\|\w^i\|_\infty$ (such as the bound from Lemma~\ref{lem:binom}) to compute the values $p_t^i$; we abuse notation here and use $\|\w^i\|_\infty$ to also denote this upper bound.} $p_t^i = \frac{w_t^i}{\|\w^i\|_\infty}$, where $\w^i = \langle w_{1}^i, w_{2}^i, \ldots, w_T^i\rangle$. At the same time the booster updates $\alpha_t^i$ as well, and then it is ready to make a prediction for the next round.

We introduce some more notation to ease the presentation. Let $z_t^i = y_t \WL^i(\x_t)$ and $s_t^i = s_t^{i-1} + \alpha_t^i z_t^i$ with $s_t^0 = 0$. Define $\z^i = \langle z_{1}^i, z_{2}^i, \ldots, z_T^i \rangle$. Finally, a martingale concentration bound using (\ref{eq:wl-guarantee}) yields the following bound (proof deferred to Appendix~\ref{app:lem:WLA}). The bound can be seen as a weighted version of (\ref{eq:wl-guarantee}) which is necessary for the rest of the analysis.
\begin{lemma}\label{lem:WLA} There is a constant $\tS = 2S + \tilde{O}(\tfrac{1}{\gamma})$ such that for any $T$, with high probability, for every weak learner $\WL^i$ we have
\[\w^{i} \cdot \z^{i} \geq \gamma \|\w^{i}\|_1 - \tS\|\w^i\|_\infty.\]
\end{lemma}

\subsection{Handling Importance Weights}

Typical online learning algorithms can handle {\em importance weighted} examples: each example $(\x_t, y_t)$ comes with a weight $p_t \in [0, 1]$, and the loss on that example is scaled by $p_t$, i.e. the loss for predicting $\hat{y}_t$ is $p_t\1\{\hat{y}_t \neq y_t\}$. Consider the following natural extension to the definition of online weak learners which incorporates importance weighted examples: we now require that for any sequence of weighted examples $(\x_t, y_t)$ with weight $p_t \in [0, 1]$ for $t = 1, 2, \ldots, T$, the online learner generates predictions $\hat{y}_t$ such that with probability at least $1 - \delta$,
\begin{equation} \label{eq:wl-guarantee-weights}
	\sum_{t=1}^T p_t\1\{\hat{y}_t \neq y_t\} \leq (\half - \gamma)\sum_{t=1}^T p_t + S.
\end{equation}
Having access to such a weak learner makes the boosting algorithm simpler: we now simply pass every example $(\x_t, y_t)$ to every weak learner $\WL^i$ using the probability $p_t^i = \frac{w_t^i}{\|\w^i\|_\infty}$ as importance weights. The advantage is that the bound (\ref{eq:wl-guarantee-weights}) immediately implies the following inequality for any weak learner $\WL^i$, which can be seen as a (stronger) analogue of Lemma~\ref{lem:WLA}.
\begin{equation} \label{eq:WLA-weights}
\w^{i} \cdot \z^{i} \geq 2\gamma \|\w^{i}\|_1 - 2S\|\w^i\|_\infty.	
\end{equation}
Since the analysis depends only on the bound in Lemma~\ref{lem:WLA}, if we use the importance-weighted version of the boosting algorithm, then we can simply use inequality (\ref{eq:WLA-weights}) instead in the analysis, which gives a slightly tighter version of Theorem~\ref{thm:main}, viz. the excess loss can now be bounded by $O(\frac{S}{\gamma})$. 

In the rest of the paper, for simplicity of exposition we assume that the $p_t^i$'s are used as sampling probabilities rather than importance weights, and give the analysis using the bound from Lemma~\ref{lem:WLA}. In experiments, however, using the $p_t^i$'s as importance weights rather than sampling probabilities led to better performance.

\subsection{Discussion of Weak Online Learning Assumption}

We now justify our definition of weak online learning, viz. inequality (\ref{eq:wl-guarantee}). In the standard batch boosting case, the corresponding weak learning assumption (see for example \citet{SchapireFr12}) made is that there is an algorithm which,  given a training set of examples and an arbitrary distribution on it, generates a hypothesis that has error at most $\half - \gamma$ on the training data under the given distribution. This statement can be interpreted as making the following two implicit assumptions:
\begin{enumerate}
	\item (Richness.) Given an edge parameter $\gamma \in (0, \half)$, there is a set of hypotheses, $\H$, such that given any training set (possibly, a multiset) of examples $U$, there is some hypothesis $h \in \H$ with error at most $\half - \gamma$, i.e. 
	\[\sum_{(\x, y) \in U} \!\!\1\{h(\x) \neq y\} \leq (\half - \gamma)|U|.\]
	
	\item (Agnostic Learnability.) For any $\epsilon \in (0, 1)$, there is an algorithm which, given any training set (possibly, a multiset) of examples $U$, can compute a nearly optimal hypothesis $h \in \H$, i.e. 
	\[\sum_{(\x, y) \in U} \!\!\1\{h(\x) \neq y\} \leq \inf_{h' \in \H} \!\! \sum_{(\x, y) \in U} \!\!\!\!\! \1\{h'(\x) \neq y\} + \epsilon |U|.\]
\end{enumerate}
Our weak online learning assumption can be seen as arising from a direct generalization of the above two assumptions to the online setting. Namely, the richness assumption stays the same, whereas the agnostic learnability of $\H$ assumption is replaced by an agnostic online learnability of $\H$ assumption (c.f. \citet{DPS}). I.e., there is an online learning algorithm which, given any sequence of examples, $(\x_t, y_t)$ for $t = 1, 2, \ldots, T$, generates predictions $\hat{y}_t$ such that
\[ \sum_{t=1}^T \1\{\hat{y}_t \neq y_t\} \leq \inf_{h \in \H} \sum_{t=1}^T \1\{h(\x_t) \neq y_t\} + R(T),\]
where $R: \mathbb{N} \rightarrow \R_+$ is the regret, a non-decreasing, sublinear function of the number of prediction periods $T$. Since online learning algorithms are typically randomized, we assume the above bound holds with high probability. The following lemma shows that richness and agnostic online learnability immediately imply our online weak learning assumption (\ref{eq:wl-guarantee}):
\begin{lemma} \label{lem:agnostic-to-weak}
	Suppose the sequence of examples $(\x_t, y_t)$ is obtained from a data set for which there exists a hypothesis class $\H$ that is both rich for edge parameter $2\gamma$ and agnostically online learnable with regret $R(\cdot)$. Then, the agnostic online learning algorithm for $\H$ satisfies the weak learning assumption (\ref{eq:wl-guarantee}), with edge $\gamma$ and excess loss $S = \max_T (R(T) - \gamma T)$.
\end{lemma}
\begin{proof}
	For the given sequence of examples $(\x_t, y_t)$ for $t = 1, 2, \ldots, T$, the richness with edge parameter $2\gamma$ and agnostic online learnability assumptions on $\H$ imply that with high probability, the predictions $\hat{y}_t$ generated by the agnostic online learning algorithm for $\H$ satisfy
	\[ \sum_{t=1}^T \1\{\hat{y}_t \neq y_t\} \leq (\half - 2\gamma) T + R(T).\]
	It only remains to show that 
	\[(\half - 2\gamma) T + R(T) \leq (\half - \gamma)T + S,\]
	or equivalently, $R(T) \leq \gamma T + S$, which is true by the definition of $S$. This concludes the proof.
\end{proof}

Various agnostic online learning algorithms are known that have a regret bound of $O(\sqrt{T \ln(\deltainv)})$; for example, a standard experts algorithm on a finite hypothesis space such as Hedge. If we use such an online learning algorithm as a weak online learner, then a simple calculation implies, via Lemma~\ref{lem:agnostic-to-weak}, that it has excess loss $\Theta(\frac{\ln(\deltainv)}{\gamma})$. Thus, by Theorem~\ref{thm:main}, we obtain an online boosting algorithm with near-optimal sample complexity.

\section{An Optimal Algorithm}\label{sec:2LevBBM}
In this section, we generalize a family of potential based batch boosting algorithms to the online setting. With a specific potential, an online version of boost-by-majority is developed with optimal number of weak learners
and near-optimal sample complexity. Matching lower bounds will be shown at the end of the section.

\subsection{A Potential Based Family and Boost-By-Majority}\label{subsec:BBM}
In the batch setting, many boosting algorithms can be understood in a unified framework called drifting games \citep{Schapire01}. Here, we generalize the analysis and propose a potential based family of online boosting algorithms.

Pick a sequence of $N+1$ non-increasing potential functions $\Phi_i(s)$ such that  
\begin{equation}\label{equ:potentials}
\begin{split}
\Phi_N(s) &\geq \1\{s \leq 0\} ,\\
\Phi_{i-1}(s) &\geq (\half-\halfgamma)\Phi_i(s-1) + (\half+\halfgamma)\Phi_i(s+1).
\end{split}
\end{equation}
Then the algorithm is simply to set $\alpha_t^i = 1$ and $w_t^i = \frac{1}{2}(\Phi_i(s_{t}^{i-1}-1) - \Phi_i(s_{t}^{i-1}+1))$.
The following theorem states the error rate bound of this general scheme.
\begin{lemma}\label{lem:potential-family}
For any $T$ and $N$, with high probability, the number of mistakes made by the algorithm described above is bounded as follows:
$$ \sum_{t=1}^T \1\{\hat y_t \neq y_t\} \leq  
\Phi_0(0)T + \tS\sum_i \|\w^i\|_\infty.$$
\end{lemma}
\begin{proof}
The key property of the algorithm is that for any fixed $i$ and $t$, 
one can verify the following:
\[
\Phi_i(s_{t}^i) + w_{t}^i (z_{t}^i-\gamma)  
= \Phi_i(s_{t}^{i-1} + z_{t}^i) + w_{t}^i (z_{t}^i-\gamma)  = (\tfrac{1}{2}-\halfgamma)\Phi_i(s_{t}^{i-1}-1) + (\tfrac{1}{2}+\halfgamma)\Phi_i(s_{t}^{i-1}+1) \leq \Phi_{i-1}(s_{t}^{i-1})
\]
by plugging the formula of $w_t^i$, realizing that $z_t^i$ can only be $-1$ or $1$,
and using the definition of $\Phi_{i-1}(s)$ from Eq. \eqref{equ:potentials}. $t = 1$ to $T$, we get
\[\sum_{t=1}^T\Phi_i(s_{t}^i) + \w^i \cdot \z^i - \gamma \|\w^i\|_1\ \leq\ \sum_{t=1}^T \Phi_{i-1}(s_{t}^{i-1}).\]
Using Lemma \ref{lem:WLA}, we get 
\[\sum_{t=1}^T \Phi_i(s_{t}^i) \ \leq\ \sum_{t=1}^T \Phi_{i-1}(s_{t}^{i-1}) + \tS\|\w^i\|_\infty.\]
which relates the sums of all examples' potential 
for two successive weak learners. We can therefore apply this inequality iteratively to arrive at:
$$\sum_{t=1}^T \Phi_N(s_{t}^N) \leq \sum_{t=1}^T \Phi_0(0) + \tS\|\w^i\|_\infty.$$
The proof is completed by noting that 
\[\Phi_N(s_{t}^N) \geq \1\{s_t^N \leq 0\} = \1\{\hat y_t \neq y_t\}\] since $y_t\hat{y}_t = \sign(s_{t}^N)$ by definition.
\end{proof}

Note that the $\tS\|\w^i\|_\infty$ term becomes a penalty for the final error rate. Therefore, we naturally want this penalty term to be relatively small. This is not necessarily true for any choice of the potential function. For example, if $\Phi_i(s)$ is the exponential potential that leads to a variant of AdaBoost in the batch setting \citep[see][Chap. 13]{SchapireFr12}, then the weight $w_t^i$ could be exponentially large.

Fortunately, there is indeed a set of potential functions that produces small weights, which, in the batch setting, corresponds to an algorithm called boost-by-majority (BBM) \cite{Freund95}. All we need to do is to let Eq. \eqref{equ:potentials} hold with equality, and direct calculation shows:
$$ \Phi_i(s) = \sum_{k=0}^{\lfloor\frac{N-i-s}{2} \rfloor}
\binom{N-i}{k} \left(\frac{1}{2}+\frac{\gamma}{2}\right)^k \left(\frac{1}{2}-\frac{\gamma}{2}\right)^{N-i-k},$$
and 
\begin{equation}\label{equ:BBM_weight}
w_{t}^i =\frac{1}{2}\binom{N-i}{k_{t}^i} \left(\frac{1}{2}+\frac{\gamma}{2}\right)^{k_{t}^i} 
\left(\frac{1}{2}-\frac{\gamma}{2}\right)^{N-i-k_{t}^i} 
\end{equation}
where $k_{t}^i = \lfloor\frac{N-i-s_{t}^{i-1}+1}{2} \rfloor$
and $\binom{n}{k}$ is defined to be $0$ if $k < 0$ or $k > n$.
In other words, imagine flipping a biased coin whose probability of heads is 
$\frac{1}{2}+\halfgamma$ for $N-i$ times.
Then $\Phi_i(s)$ is exactly the probability of seeing at most $(N-i-s)/2$ heads
and $w_{t}^i$ is half of the probability of seeing $k_{t}^i$ heads.
We call this algorithm {\it Online BBM}. The pseudocode is given in Algorithm~\ref{alg:OnlineBBM}.
\begin{algorithm}[ht]
\caption{Online BBM}
\label{alg:OnlineBBM}
\begin{algorithmic}[1]
\FOR{$t=1$ {\bfseries to} $T$}
    \STATE Receive example $\x_t$.   
    \STATE Predict $\hat{y}_t = \sign(\sum_{i=1}^{N} \WL^i(\x_t))$, receive label $y_t$. 
    \STATE Set $s_t^0 = 0$.
    \FOR{$i=1$ {\bfseries to} $N$}
    	\STATE Set $s_{t}^{i} = s_{t}^{i-1} + y_t \WL^i(\x_t)$.
        \STATE Set $k_{t}^i = \lfloor\frac{N-i-s_{t}^{i-1}+1}{2} \rfloor$.
        \STATE Set $w_{t}^i = \binom{N-i}{k_{t}^i} \left(\frac{1}{2}+\frac{\gamma}{2}\right)^{k_{t}^i} 
          \left(\frac{1}{2}-\frac{\gamma}{2}\right)^{N-i-k_{t}^i}$.
        \STATE Pass training example $(\x_t, y_t)$ to $\WL^i$ with probability $p_t^i = \frac{w_t^i}{\|\w^i\|_\infty}$.
    \ENDFOR
\ENDFOR
\end{algorithmic}
\end{algorithm}

One can see that the weights produced by this algorithm are small since trivially $w_t^i \leq 1/2$. To get a better result, however, we need a better estimate of $\|\w^i\|_\infty$ stated in the following lemma.
\begin{lemma}\label{lem:binom}
If $w_t^i$ is defined as in Eq. \eqref{equ:BBM_weight}, then we have
$w_t^i = O(1/\sqrt{N-i})$ for any $i < N$.
\end{lemma}

This lemma was essentially proven before by \citet[Lemma 2.3.10]{Freund93}.
We give an alternative and simpler proof in Appendix \ref{app:lem:binom} in the supplementary material by using Berry-Esseen theorem directly. We are now ready to state the main results of Online BBM.

\begin{theorem}\label{thm:BBM}
For any $T$ and $N$, with high probability, the number of mistakes made by the Online BBM algorithm is bounded as follows:
\begin{equation}\label{equ:BBM_error}
\exp(-\tfrac{1}{2}N\gamma^2)T + \tilde{O}(\sqrt{N}(S + \tfrac{1}{\gamma})).
\end{equation}
Thus, in order to achieve error rate $\epsilon$,
it suffices to use $N = \Theta(\frac{1}{\gamma^2}\ln\frac{1}{\epsilon})$ weak learners, which gives an excess loss bound of
$\tilde\Theta(\frac{S}{\gamma} + \frac{1}{\gamma^2})$.
\end{theorem}
\begin{proof}
A direct application of Hoeffding's inequality gives $\Phi_0(0) \leq \exp(-N\gamma^2/2)$.
With Lemma \ref{lem:binom} we have
\[\sum_{i} \|\w^i\|_\infty =  O\left(\sum_{i=1}^{N-1} \frac{1}{\sqrt{N-i}}\right) = O(\sqrt{N}).\]
Applying Lemma~\ref{lem:potential-family} proves Eq. \eqref{equ:BBM_error}.
Now if we set $N = \frac{2}{\gamma^2}\ln\frac{1}{\epsilon}$,
then 
\[\sum_{t=1}^T \1\{\hat y_t \neq y_t\} \leq \epsilon T + \tilde{O}(\sqrt{N}(S + \tfrac{1}{\gamma})) = \epsilon T + \tilde{O}(\tfrac{S}{\gamma} + \tfrac{1}{\gamma^2}).\]
\end{proof}

\subsection{Matching Lower Bounds}\label{subsec:lower_bounds}
We give lower bounds for the number of weak learners and the sample complexity in this section that show that our Online BBM algorithm is optimal up to logarithmic factors.
\begin{theorem} \label{thm:lowerbounds}
For any $\gamma \in (0, \tfrac{1}{4})$, $S \geq \frac{\ln(\deltainv)}{\gamma}$, $\delta \in (0, 1)$ and $\epsilon \in (0, 1)$, there is a weak online learning algorithm with edge $\gamma$ and excess loss $S$ satisfying (\ref{eq:wl-guarantee}) with probability at least $1-\delta$, such that to achieve error rate $\epsilon$, an online boosting algorithm needs at least
$\Omega(\frac{1}{\gamma^2}\ln\frac{1}{\epsilon})$ weak learners
and a sample complexity of $\Omega(\frac{S}{\epsilon\gamma}) = \Omega(\frac{1}{\epsilon}(\frac{S}{\gamma} + \frac{1}{\gamma^2}))$.
\end{theorem}
\begin{proof}
The proof of both lower bounds use a similar construction. In either case, all examples' labels are generated uniformly at random from $\{-1,1\}$, and in time period $t$, each weak learner outputs the correct label $y_t$ independently of all other weak learners and other examples with a certain probability $p_t$ to be specified later. Thus, for any $T$, by the Azuma-Hoeffding inequality, with probability at least $1-\delta$, the predictions $\hat{y}_t$ made by the weak learner satisfy
\begin{equation} \label{eq:azuma-am-gm}
	\sum_{t=1}^T \1\{y_t \neq \hat{y}_t\} \leq \sum_{t=1}^T (1-p_t) + \sqrt{2T \ln(\deltainv)} \leq \sum_{t=1}^T (1-p_t) + \gamma T + \frac{\ln(\deltainv)}{2\gamma} 
\end{equation}
where the last inequality follows by the arithmetic mean-geometric mean inequality. We will now carefully choose $p_t$ so that inequality (\ref{eq:azuma-am-gm}) implies inequality (\ref{eq:wl-guarantee}).

For the lower bound on the number of weak learners, we set $p_t = \frac{1}{2}+2\gamma$, so that inequality (\ref{eq:azuma-am-gm}) implies that with probability at least $1-\delta$, the predictions $\hat{y}_t$ made by the weak learner satisfy
\[ \sum_{t=1}^T \1\{y_t \neq \hat{y}_t\} \leq (\half - \gamma)T + \frac{\ln(\deltainv)}{2\gamma} \leq (\half - \gamma)T + S.\]
Thus, the weak online learner has edge $\gamma$ with excess loss $S$. In this case, the Bayes optimal output of a booster using $N$ weak learners is to simply take a majority vote of all the weak learners \citep[see for instance][Chap. 13.2.6]{SchapireFr12}, and the probability that the majority vote is incorrect is $ \Theta(\exp(-8N\gamma^2))$. Setting this error to $\epsilon$ and solving for $N$ gives the desired lower bound.

Now we turn to the lower bound on the sample complexity. We divide the whole process into two phases: for $t \leq T_0 = \frac{S}{4\gamma}$, we set $p_t = \half$, and for $t > T_0$, we set $p_t = \frac{1}{2}+2\gamma$. Now, if $T \leq T_0$, inequality (\ref{eq:azuma-am-gm}) implies that with probability at least $1-\delta$, the predictions $\hat{y}_t$ made by the weak learner satisfy
\begin{equation} \label{eq:sample-complexity-I}
	\sum_{t=1}^T \1\{y_t \neq \hat{y}_t\} \leq (\half + \gamma) T +  \frac{\ln(\deltainv)}{2\gamma} \leq (\half - \gamma)T + S
\end{equation}
using the fact that $T \leq T_0 = \frac{S}{4\gamma}$ and $S \geq \frac{\ln(\deltainv)}{\gamma}$. Next, if $T > T_0$, let $T' = T - T_0$, and again inequality (\ref{eq:azuma-am-gm}) implies that with probability at least $1-\delta$, the predictions $\hat{y}_t$ made by the weak learner satisfy
\begin{equation}
	\sum_{t=1}^T \1\{y_t \neq \hat{y}_t\} \leq \half T_0 + (\half - 2\gamma)T' + \gamma T + \frac{\ln(\deltainv)}{2\gamma} = (\half - \gamma)T + 2\gamma T_0 + \frac{\ln(\deltainv)}{2\gamma} \leq (\half - \gamma)T + S, \label{eq:sample-complexity-II}
\end{equation}
since $S \geq \frac{\ln(\deltainv)}{\gamma}$. Inequalities (\ref{eq:sample-complexity-I}) and (\ref{eq:sample-complexity-II}) imply that the weak online learner has edge $\gamma$ with excess loss $S$.

However, in the first phase (i.e. $t \leq T_0$), since the predictions of the weak learners are uncorrelated with the true labels, it is clear that no matter what the booster does, it makes a mistake with probability $\half$. Thus, it will make $\Omega(T_0)$ mistakes with high probability in the first phase, and thus to achieve $\epsilon$ error rate, it needs at least $\Omega(T_0/\epsilon) = \Omega(\frac{S}{\epsilon\gamma})$ examples.
\end{proof}

\section{An Adaptive Algorithm}\label{sec:AdaBoost.OL}

Although the Online BBM algorithm is optimal, it is unfortunately not adaptive since it requires the knowledge of $\gamma$ as a parameter, which is unknown ahead of time. As discussed in the introduction, adaptivity is essential to the practical performance of boosting algorithms such as AdaBoost.

In this section we thus study adaptive online boosting algorithms using the theory of online loss minimization as the main tool. It is known that boosting can be viewed as trying to find a linear combination of weak hypotheses to minimize the total loss of the training examples, usually using functional gradient descent \citep[see for details][Chap. 7]{SchapireFr12}. AdaBoost, for instance, minimizes the exponential loss. Here, as discussed before, we intuitively want to avoid using exponential loss since it could lead to large weights. Instead, we will consider logistic loss $\ell(s) = \ln(1+\exp(-s))$, which results in an algorithm called AdaBoost.L in the batch setting \citep[Chap. 7]{SchapireFr12}.

In the online setting, we conceptually define $N$ different ``experts'' giving advice on what to predict on the current example $\x_t$. In round $t$, expert $i$ predicts by combining the first $i$ weak learners:
$\hat{y}_t^i = \sign(\sum_{j=1}^i \alpha_t^j \WL^j(\x_t))$. Now as in AdaBoost.L, the weight $w_t^i$ for $\WL^i$ is obtained by computing the logistic loss of the prediction of expert $i-1$, i.e. $\ell(s_t^{i-1})$, and then setting $w_t^i$ to be the negative derivative of the loss:
\[w_{t}^i = -\ell'(s_t^{i-1}) = \frac{1}{1+\exp(s_t^{i-1})} \in [0,1].\]
In terms of the weight of $\WL^i$, i.e. $\alpha_t^i$, ideally we wish to mimic AdaBoost.L and use a fixed $\alpha^i$ for all $t$ such that the total logistic loss is minimized:
$\alpha^i = \arg\min_\alpha \sum_{t=1}^T \ell(s_t^{i-1} + \alpha z_t^i)$.
Of course this is not possible because $\alpha^i$ depends on the future unknown examples. Nevertheless, it turns out that we can almost achieve that using tools from online learning theory. Indeed, one of the fundamental topics in online learning is exactly how to perform almost as well as the best fixed choice ($\alpha^i$) in the hindsight. 

Specifically, it turns out that it suffices to restrict $\alpha$ to the feasible set $[-2, 2]$. Then consider the following simple one dimensional online learning problem: on each round $t$, algorithm predicts $\alpha_{t}^i$ from a feasible set $[-2, 2]$; the environment then reveals loss function $f_t(\alpha) = \ell(s_t^{i-1} + \alpha z_t^i)$ and the algorithm suffers loss $f_t(\alpha_t^i)$. There are many so-called ``low-regret" algorithms in the literature (see the survey by \citet{Shalevshwartz11}) for this problem ensuring 
$$ \sum_{t=1}^T f_t(\alpha_t^i) - \min_{\alpha \in [-2, 2]} \sum_{t=1}^T f_t(\alpha) \leq R_{T}^i,$$
where $R_T^i$ is sublinear in $T$ so that on average it goes to $0$ when $T$ is large and the algorithm is thus doing almost as well as the best constant choice $\alpha^i$. The simplest low-regret algorithm in this case is perhaps {\it online gradient descent} 
\cite{Zinkevich03}:
$$\alpha_{t+1}^i  = \Pi\left(\alpha_t^i - \eta_t f_t'(\alpha_t^i)\right) 
= \Pi\left(\alpha_t^i + \frac{\eta_t z_t^i}{1 + \exp(s_t^{i})}\right), $$
where $\eta_t$ is a time-varying learning rate
and $\Pi$ represents projection onto the set $[-2, 2]$, i.e., 
$\Pi(\cdot) = \max\{-2, \min\{2, \cdot\}\}$.
Since the loss function is actually $1$-Lipschitz ($|f'_t(\alpha)| \leq 1$),
if we set $\eta_t$ to be $4/\sqrt{t}$, 
then standard analysis shows $R_T^{i} = 4\sqrt{T}$.

Finally, it remains to specify the algorithm's final prediction $\hat{y}_t$.
In Online BBM, we simply used the advice of expert $N$. Unfortunately the algorithm described in this section cannot guarantee that expert $N$ will always make highly accurate predictions. However, as we will show in the proof of Theorem \ref{thm:AdaBoost.OL}, the algorithm does ensure that at least {\it one of the $N$ experts} will have high accuracy. Therefore, what we really need to do is to decide which expert to follow on each round, and try to predict almost as well as the best fixed expert in the hindsight. This is again another classic online learning problem (called expert or hedge problem),
and can be solved, for instance, by the well-known Hedge algorithm \citep{LittlestoneWa94, FreundSc97}. The idea is to pick an expert on each round randomly with different importance weights according to their previous performance.

We call the final resulting algorithm AdaBoost.OL (O stands for online and L stands for logistic loss), and summarize it in Algorithm \ref{alg:AdaBoost.OL}. Note that as promised, AdaBoost.OL is an adaptive online boosting algorithm and does not require knowing $\gamma$ in advance. In fact, in the analysis we do not even assume that the weak learners satisfy the bound (\ref{eq:wl-guarantee}). Instead, define the quantities $\gamma_i \triangleq \frac{\w^i \cdot \z^i}{2\|\w^i\|_1}$ for each weak learner $\WL^i$. This can be interpreted as the (weighted) edge over random guessing that $\WL^i$ obtains. Note that $\gamma_i$ may even be negative, which means flipping the sign of $\WL^i$'s predictions performs better than random guessing. Nevertheless, the algorithm can still make accurate predictions even with negative $\gamma_i$ since it will end up choosing negative weights $\alpha_t^i$ in that case. The performance of AdaBoost.OL is provided below.

\begin{algorithm}[tb]
\caption{AdaBoost.OL}
\label{alg:AdaBoost.OL}
\begin{algorithmic}[1]
\STATE {\bfseries Initialize:} 
$\forall i: v_1^i = 1, \alpha_1^i = 0$.
\FOR{$t=1$ {\bfseries to} $T$}
    \STATE Receive example $\x_t$.   
    \FOR{$i=1$ {\bfseries to} $N$}
        \STATE Set $\hat{y}_t^i = \sign(\sum_{j=1}^{i} \alpha_t^j \WL^j(\x_t))$.
    \ENDFOR
    \STATE Randomly pick $i_t$ with $\Pr[i_t = i] \propto v_t^i$.  
    \STATE Predict $\hat{y}_t = \hat{y}_t^{i_t}$, receive label $y_t$. 
    \STATE Set $s_t^0 = 0$.
    \FOR{$i=1$ {\bfseries to} $N$}
    	\STATE Set $z_t^i = y_t \WL^i(\x_t)$.
       	\STATE Set $s_{t}^{i} = s_{t}^{i-1} + \alpha_t^i z_t^i$.
        \STATE Set $\alpha_{t+1}^i = \Pi\left(\alpha_t^i + \frac{\eta_t z_t^i}{1 + \exp(s_t^{i})}\right)$ with $\eta_t = 4/\sqrt{t}$.
        \STATE Pass example $(\x_t, y_t)$ to $\WL^i$ with probability\footnote{Note that we are using the bound $\|\w^i\|_\infty \leq 1$ here.} $p_t^i = w_{t}^i = 1/(1+\exp(s_t^{i-1}))$.
        \STATE Set $v_{t+1}^i = v_{t}^i \cdot \exp(-\1\{y_t \neq \hat{y}_t^i\})$. 
    \ENDFOR
\ENDFOR
\end{algorithmic}
\end{algorithm}

\begin{theorem}\label{thm:AdaBoost.OL}
For any $T$ and $N$, with high probability, the number of mistakes made by AdaBoost.OL is bounded by
\[\frac{2T}{\sum_i\gamma_i^2}  + \tilde{O}\left(\frac{N^2}{\sum_i\gamma_i^2}\right).\]
\end{theorem}
\begin{proof}
Let the number of mistakes made by expert $i$ be $M_i \triangleq \sum_{t=1}^T \1\{y_t \neq \hat{y}_t^i\}$, also define $M_0 = T$ for convenience. Note that AdaBoost.OL is using a variant of the Hedge algorithm with $\1\{y_t \neq \hat{y}_t^i\}$ being the loss of expert $i$ on round $t$ (Line 7 and 15). So by standard analysis \citep[see e.g.][Corollary 2.3]{CesabianchiLu06}, and the Azuma-Hoeffding inequality, we have with high probability
\begin{equation} \label{eq:mw}
	\sum_{t=1}^T \1\{y_t \neq \hat{y}_t \}  \leq 2\min_i M_i + 2 \ln(N) + \tilde{O}(\sqrt{T}).
\end{equation}

Now, whenever expert $i-1$ makes a mistake (i.e. $s_t^{i-1} \leq 0$),
we have $w_t^i = 1/(1+\exp(s_t^{i-1})) \geq 1/2$ and therefore 
\begin{equation}\label{equ:lower_bound_weights}
\|\w^i\|_1\geq M_{i-1}/2.
\end{equation}
Note that Eq. \eqref{equ:lower_bound_weights} holds even for $i=1$ by the definition of $M_0$.
We now bound the difference between the logistic loss of two successive experts,
$\Delta_i \triangleq \sum_{t=1}^T \left(\ell(s_t^i) - \ell(s_t^{i-1})\right)$.
Online gradient descent (Line 13) ensures that
\begin{equation} \label{eq:ogd}
	\sum_{t=1}^T \ell(s_t^i)  \leq \min\limits_{\alpha\in [-2,2]} 
\sum_{t=1}^T \ell(s_t^{i-1} + \alpha z_t^i) + 4\sqrt{T},
\end{equation}
as discussed previously.
On the other hand, direct calculation shows
$ \ell(s_t^{i-1} + \alpha z_t^i) - \ell(s_t^{i-1}) = \ln\left(1 + w_t^i (e^{-\alpha z_t^i}-1)\right)
\leq w_t^i (e^{-\alpha z_t^i}-1) $.
With $\sigma_i \triangleq \sum_{t=1}^T \frac{w_t^i}{\|\w^i\|_1} \1\{z_t^i = 1\} = \half + \gamma_i$, we thus have
\begin{align}
\min\limits_{\alpha\in [-2, 2]}\sum_{t=1}^T \left(\ell(s_t^{i-1} + \alpha z_t^i) - \ell(s_t^{i-1})\right) &\leq \min\limits_{\alpha\in [-2, 2]}\|\w^i\|_1(\sigma_i e^{-\alpha} + (1-\sigma_i) e^\alpha - 1) \notag \\
&\leq -\tfrac{1}{2}\|\w^i\|_1(2\sigma_i - 1)^2 \label{eq:ineq1}\\
&= -2\gamma_i^2\|\w^i\|_1 \label{eq:ineq2}\\
&\leq -\gamma_i^2 M_{i-1}. \label{eq:ineq3}
\end{align}
Here, inequality (\ref{eq:ineq1}) follows from Lemma~\ref{lem:inequality} and inequality (\ref{eq:ineq3}) from inequality (\ref{equ:lower_bound_weights}).
The above inequality and inequality (\ref{eq:ogd}) imply that
\[ \Delta_i \leq -\gamma_i^2 M_{i-1} + 4\sqrt{T}.\]
Summing over $i = 1, \ldots, N$ and rearranging gives 
\[ \sum_{i=1}^N \gamma_i^2 M_{i-1} + \sum_{t=1}^T \ell(s_t^N) \leq \sum_{t=1}^T \ell(0) + 4N\sqrt{T} \]
which implies that
\[ \min_i M_i \leq \min_i M_{i-1} \leq \frac{\ln(2)}{\sum_i \gamma_i^2} T + \frac{4N}{\sum_i\gamma_i^2}\sqrt{T}\]
since $M_i \leq M_0$ for all $i$, $\ell(s_t^N) \geq 0$ for all $t$ and $\ell(0) = \ln(2)$. Using this bound in inequality (\ref{eq:mw}), we get
\[
\sum_{t=1}^T \1\{y_t \neq \hat{y}_t \} \leq \frac{2 \ln(2) T}{\sum_i\gamma_i^2} + \tilde{O}\left(\frac{N\sqrt{T}}{\sum_i\gamma_i^2} + \ln(N)\right) \leq \frac{2T}{\sum_i\gamma_i^2}  + \tilde{O}\left(\frac{N^2}{\sum_i\gamma_i^2}\right),
\]
where the last inequality follows from the bound $\frac{cN\sqrt{T}}{\sum_i\gamma_i^2} \leq \frac{T}{2\sum_i\gamma_i^2} + \frac{c^2N^2}{2\sum_i\gamma_i^2}$, where $c$ is the hidden $\tilde{O}(1)$ factor in the $\tilde{O}(\frac{N\sqrt{T}}{\sum_i\gamma_i^2})$ term, using the arithmetic mean-geometric mean inequality. 
\end{proof}

For the case when the weak learners do satisfy the bound (\ref{eq:wl-guarantee}), we get the following bound on the number of errors:
\begin{theorem}
	If the weak learners satisfy (\ref{eq:wl-guarantee}), then for any $T$ and $N$, with high probability, the number of mistakes made by AdaBoost.OL is bounded by
	\[  \frac{8T}{\gamma^2 N}  + \tilde{O}\left(\frac{N}{\gamma^2} + \frac{S}{\gamma}\right),\]
	Thus, in order to achieve error rate $\epsilon$, it suffices to use $N \geq \frac{8}{\epsilon\gamma^2}$ weak learners, which gives an excess loss bound of $\tilde{O}(\frac{S}{\gamma} + \frac{1}{\epsilon \gamma^4})$.
\end{theorem}
\begin{proof}
	The proof is on the same lines as that of Theorem~\ref{thm:AdaBoost.OL}. The only change is that in inequality (\ref{eq:ineq2}), we use the bound 
	$\gamma_i^2 \geq \frac{\gamma^2}{4} - \frac{\gamma \tS}{2\|\w^i\|_1}$ which follows from Lemma~\ref{lem:WLA} using the fact that $a \geq b - c$ implies $a^2 \geq b^2 - 2bc$ for non-negative $a, b$ and $c$, and the fact that $\|\w^i\|_\infty \leq 1$. This leads to the following change in inequality (\ref{eq:ineq3}):
	\[ \min\limits_{\alpha\in [-2, 2]}\sum_{t=1}^T \left(\ell(s_t^{i-1} + \alpha z_t^i) - \ell(s_t^{i-1})\right)  \leq -\frac{\gamma^2}{4}M_{i-1} + \gamma \tS.\]
	Continuing using this bound in the proof and simplifying, we get the stated bound on the number of errors.
\end{proof}

The following lemma is a simple calculation:
\begin{lemma} \label{lem:inequality}
	For any $\sigma \in [0, 1]$, 
	\[\min\limits_{\alpha\in [-2, 2]} \sigma e^{-\alpha} + (1-\sigma) e^\alpha \leq 1 - \half(2\sigma - 1)^2. \]
\end{lemma}
\begin{proof}
	It suffice to prove the bound for $\sigma \geq \half$; the bound for $\sigma < \half$ follows by simply using the bound for $1-\sigma$.	For $\sigma \in [0.5, 0.95]$, setting $\alpha = \half \ln(\frac{\sigma}{1-\sigma}) \in [-2, 2]$ gives 
	\[\sigma e^{-\alpha} + (1-\sigma) e^\alpha = \sqrt{4\sigma(1-\sigma)} \leq 1 - \half(2\sigma - 1)^2,\]
	since $\sqrt{1 - x} \leq 1 - \half x$ for $x \in [0, 1]$. For $\sigma \in (0.95, 1]$, setting $\alpha = \half \ln(\frac{0.95}{0.05}) \in [-2, 2]$ we have
	\begin{align*}
		&\sigma e^{-\alpha} + (1-\sigma) e^\alpha \leq 0.95 e^{-\alpha} + 0.05 e^{\alpha} = \sqrt{0.19} \\
		&\leq \half \leq 1 - \half(2\sigma - 1)^2.
	\end{align*}
\end{proof}

Although the number of weak learners and excess loss for Adaboost.OL are  suboptimal, the adaptivity of AdaBoost.OL is an appealing feature and leads to good performance in experiments. The possibility of obtaining an algorithm that is both adaptive and optimal is left as an open question.

\section{Experiments}\label{sec:experiments}

While the focus of this paper is a theoretical investigation of online boosting, we have also performed experiments to evaluate our algorithms.

We extended the Vowpal Wabbit open source machine learning system~\cite{VW} to include the algorithms studied in this paper. We used VW's default base learning algorithm as our weak learner, tuning only the learning rate. The online boosting algorithms implemented were Online BBM, AdaBoost.OL, OSBoost (using uniform weighting on the weak learners) and OSBoost.OCP from~\citep{ChenLiLu12}, all using importance weighted examples in VW. We also implemented AdaBoost.OL.S, which is the version of AdaBoost.OL where examples sent to VW are sampled rather than weighted.

All experiments were done on a diverse collection of 13 publically available
datasets. The datasets come from the UCI repository, KDD Cup challenges, and the HCRC Map Task Corpus. A description of these datasets is given in Table~\ref{tab:datasets}.
\begin{table}[t] 
\caption{Below, $d$ is the number of unique features in the dataset, 
and $s$ is the average number of features per example.}
\label{tab:datasets}
\begin{center}
\begin{tabular}{|c||r|c|r|}

\hline
Dataset & instances & $s$ & $d$ \\
\hline
\hline
20news &  18,845  &  93.9  & 101,631
\\
a9a &  48,841  &  13.9  & 123
\\
activity &  165,632  &  18.5  & 20
\\
adult &  48,842  &  12.0  & 105
\\
bio &  145,750  &  73.4  & 74
\\
census &  299,284  &  32.0  & 401
\\
covtype &  581,011  &  11.9  & 54
\\
letter &  20,000  &  15.6 & 16
\\
maptaskcoref &  158,546  &  40.4  & 5,944
\\
nomao &  34,465  &  82.3  & 174
\\
poker&  946,799  &  10.0  & 10
\\
rcv1&  781,265  &  75.7  & 43,001
\\
vehv2binary &  299,254  &  48.6  & 105  \\
\hline
\end{tabular}
\end{center}
\end{table}
For each dataset, we performed a random split with 80\% of the data used for training and the remaining 20\% for testing. We tuned the learning rate, the number of weak learners, and the edge parameter $\gamma$ (for all but AdaBoost.OL) using progressive validation 0-1 loss on the training set. Reported is the 0-1 loss on the test set.

\begin{table}[!ht] 
\caption{Performance of various online boosting algorithms on various datasets. The lowest loss attained for each dataset is bolded. The baseline is the loss obtained by running the weak learner, VW, on the data.}
\label{tab:exp-results}
\begin{center}
{\small
\begin{tabular}{|c|c|c|c|c|c|c|}

\hline

Dataset & VW baseline & Online BBM & AdaBoost.OL & AdaBoost.OL.S & OSBoost.OCP & OSBoost \\
\hline
\hline
20news   & 0.0812 & {\bf 0.0775} & 0.0777 & 0.0777 & 0.0791 & 0.0801 \\
a9a      & 0.1509  & {\bf 0.1495}  & 0.1497  & 0.1497 & 0.1509 & 0.1505 \\
activity & 0.0133 & {\bf 0.0114} & 0.0128 & 0.0127 & 0.0130 & 0.0133 \\
adult    & 0.1543  & {\bf 0.1526} & 0.1536  & 0.1536 & 0.1539 & 0.1544 \\
bio      & 0.0035 & {\bf 0.0031} & 0.0032 & 0.0032 & 0.0033 & 0.0034 \\
census   & 0.0471 & {\bf 0.0469} & {\bf 0.0469} & {\bf 0.0469} & {\bf 0.0469} & 0.0470 \\
covtype  & 0.2563 &  {\bf 0.2347}  & 0.2495 & 0.2450 & 0.2470 & 0.2521 \\
letter &  0.2295    & {\bf 0.1923} & 0.2078 & 0.2078 & 0.2148 & 0.2150 \\
maptaskcoref & 0.1091 & {\bf 0.1077} & 0.1083 & 0.1083 & 0.1093 & 0.1091 \\
nomao    & 0.0641 & {\bf 0.0627} & 0.0635 & 0.0635 & {\bf 0.0627} & 0.0633 \\
poker    & 0.4555 & {\bf 0.4312} & 0.4555 & 0.4555 & 0.4555 & 0.4555 \\
rcv1&   0.0487 & 0.0485 & {\bf 0.0484} & {\bf 0.0484} & 0.0488 & 0.0488 \\
vehv2binary & 0.0292 & 0.0286 & 0.0291 & 0.0291 & {\bf 0.0284} & 0.0286\\
\hline
\end{tabular}}
\end{center}
\end{table}

It should be noted that the VW baseline is already a strong learner. The results obtained are given in Table~\ref{tab:exp-results}. As can be seen, for most datasets, Online BBM had the best performance. The average improvement of Online BBM over the baseline was 5.14\%. For AdaBoost.OL, it was
2.57\%. Using sampling in AdaBoost.OL (i.e. AdaBoost.OL.S) boosts the average to 2.67\%.  The average improvement for OSBoost.OCP was 1.98\%, followed by OSBoost with 1.13\%.

\bibliography{ref}

\begin{thebibliography}{25}
\providecommand{\natexlab}[1]{#1}
\providecommand{\url}[1]{\texttt{#1}}
\expandafter\ifx\csname urlstyle\endcsname\relax
  \providecommand{\doi}[1]{doi: #1}\else
  \providecommand{\doi}{doi: \begingroup \urlstyle{rm}\Url}\fi

\bibitem[Barak et~al.(2009)Barak, Hardt, and Kale]{BarakHaKa09}
Boaz Barak, Moritz Hardt, and Satyen Kale.
\newblock The uniform hardcore lemma via approximate bregman projections.
\newblock In \emph{The twentieth Annual ACM-SIAM Symposium on Discrete
  Algorithms}, pages 1193--1200, 2009.

\bibitem[Bartlett et~al.(2008)Bartlett, Dani, Hayes, Kakade, Rakhlin, and
  Tewari]{BDHKRT}
Peter~L. Bartlett, Varsha Dani, Thomas Hayes, Sham Kakade, Alexander Rakhlin,
  and Ambuj Tewari.
\newblock High-probability regret bounds for bandit online linear optimization.
\newblock In \emph{Proceedings of the 21st Annual Conference on Learning Theory
  (COLT 2008)}, pages 335--342, 2008.

\bibitem[Ben-David et~al.(2009)Ben-David, P\'{a}l, and Shalev-Shwartz]{DPS}
Shai Ben-David, D\'{a}vid P\'{a}l, and Shai Shalev-Shwartz.
\newblock {Agnostic Online Learning}.
\newblock In \emph{COLT 2009}, 2009.

\bibitem[Bradley and Schapire(2008)]{BradleySc08}
Joseph~K. Bradley and Robert~E. Schapire.
\newblock {FilterBoost}: Regression and classification on large datasets.
\newblock In \emph{Advances in Neural Information Processing Systems 20}, 2008.

\bibitem[Bshouty and Gavinsky(2003)]{BshoutyGa03}
Nader~H Bshouty and Dmitry Gavinsky.
\newblock On boosting with polynomially bounded distributions.
\newblock \emph{The Journal of Machine Learning Research}, 3:\penalty0
  483--506, 2003.

\bibitem[Cesa-Bianchi and Lugosi(2006)]{CesabianchiLu06}
Nicol\`o Cesa-Bianchi and G\'{a}bor Lugosi.
\newblock \emph{Prediction, Learning, and Games}.
\newblock Cambridge University Press, 2006.

\bibitem[Chen et~al.(2012)Chen, Lin, and Lu]{ChenLiLu12}
Shang-Tse Chen, Hsuan-Tien Lin, and Chi-Jen Lu.
\newblock {An Online Boosting Algorithm with Theoretical Justifications}.
\newblock In \emph{Proceedings of the 29th International Conference on Machine
  Learning}, 2012.

\bibitem[Chen et~al.(2014)Chen, Lin, and Lu]{ChenLiLu14}
Shang-Tse Chen, Hsuan-Tien Lin, and Chi-Jen Lu.
\newblock {Boosting with Online Binary Learners for the Multiclass Bandit
  Problem}.
\newblock In \emph{Proceedings of the 31st International Conference on Machine
  Learning}, 2014.

\bibitem[Freund(1992)]{Freund92}
Yoav Freund.
\newblock An improved boosting algorithm and its implications on learning
  complexity.
\newblock In \emph{Proceedings of the Fifth Annual ACM Workshop on
  Computational Learning Theory}, pages 391--398, July 1992.

\bibitem[Freund(1993)]{Freund93}
Yoav Freund.
\newblock \emph{Data Filtering and Distribution Modeling Algorithms for Machine
  Learning}.
\newblock PhD thesis, University of California at Santa Cruz, 1993.

\bibitem[Freund(1995)]{Freund95}
Yoav Freund.
\newblock Boosting a weak learning algorithm by majority.
\newblock \emph{Information and Computation}, 121\penalty0 (2):\penalty0
  256--285, 1995.

\bibitem[Freund and Schapire(1997)]{FreundSc97}
Yoav Freund and Robert~E. Schapire.
\newblock A decision-theoretic generalization of on-line learning and an
  application to boosting.
\newblock \emph{Journal of Computer and System Sciences}, 55\penalty0
  (1):\penalty0 119--139, August 1997.

\bibitem[Grabner and Bischof(2006)]{GrabnerBi06}
Helmut Grabner and Horst Bischof.
\newblock On-line boosting and vision.
\newblock In \emph{CVPR}, volume~1, pages 260--267, 2006.

\bibitem[Grabner et~al.(2008)Grabner, Leistner, and Bischof]{GrabnerLeBi08}
Helmut Grabner, Christian Leistner, and Horst Bischof.
\newblock Semi-supervised on-line boosting for robust tracking.
\newblock In \emph{ECCV}, pages 234--247, 2008.

\bibitem[Littlestone and Warmuth(1994)]{LittlestoneWa94}
Nick Littlestone and Manfred~K. Warmuth.
\newblock The weighted majority algorithm.
\newblock \emph{Information and Computation}, 108:\penalty0 212--261, 1994.

\bibitem[Liu and Yu(2007)]{LiuYu07}
Xiaoming Liu and Ting Yu.
\newblock Gradient feature selection for online boosting.
\newblock In \emph{ICCV}, pages 1--8, 2007.

\bibitem[Luo and Schapire(2014)]{LuoSc14b}
Haipeng Luo and Robert~E. Schapire.
\newblock {A Drifting-Games Analysis for Online Learning and Applications to
  Boosting}.
\newblock In \emph{Advances in Neural Information Processing Systems 27}, 2014.

\bibitem[Mason et~al.(2000)Mason, Baxter, Bartlett, and Frean]{MasonBaBaFr99b}
Llew Mason, Jonathan Baxter, Peter Bartlett, and Marcus Frean.
\newblock Functional gradient techniques for combining hypotheses.
\newblock In \emph{Advances in Large Margin Classifiers}. MIT Press, 2000.

\bibitem[Oza and Russell(2001)]{OzaRu01}
Nikunj~C. Oza and Stuart Russell.
\newblock Online bagging and boosting.
\newblock In \emph{Eighth International Workshop on Artificial Intelligence and
  Statistics}, pages 105--112, 2001.

\bibitem[Schapire(2001)]{Schapire01}
Robert~E. Schapire.
\newblock Drifting games.
\newblock \emph{Machine Learning}, 43\penalty0 (3):\penalty0 265--291, June
  2001.

\bibitem[Schapire and Freund(2012)]{SchapireFr12}
Robert~E. Schapire and Yoav Freund.
\newblock \emph{Boosting: Foundations and Algorithms}.
\newblock MIT Press, 2012.

\bibitem[Servedio(2003)]{Servedio03}
Rocco~A. Servedio.
\newblock Smooth boosting and learning with malicious noise.
\newblock \emph{Journal of Machine Learning Research}, 4:\penalty0 633--648,
  2003.

\bibitem[Shalev-Shwartz(2011)]{Shalevshwartz11}
Shai Shalev-Shwartz.
\newblock Online learning and online convex optimization.
\newblock \emph{Foundations and Trends in Machine Learning}, 4\penalty0
  (2):\penalty0 107--194, 2011.

\bibitem[VW()]{VW}
VW.
\newblock URL \url{https://github.com/JohnLangford/vowpal_wabbit/}.

\bibitem[Zinkevich(2003)]{Zinkevich03}
Martin Zinkevich.
\newblock Online convex programming and generalized infinitesimal gradient
  ascent.
\newblock In \emph{Proceedings of the Twentieth International Conference on
  Machine Learning}, 2003.

\end{thebibliography}
\bibliographystyle{plainnat}

\appendix

\section{Proof of Lemma \ref{lem:WLA}}\label{app:lem:WLA}
\begin{proof}
	Fix a weak learner, say $\WL^i$. Let 
	\[U = \{t: (\x_t, y_t) \text{ passed to } \WL^i\}.\]	
	Since inequality (\ref{eq:wl-guarantee}) holds even for {\em adaptive} adversaries, with high probability we have
	\begin{equation} \label{eq:mistake-bound}
		\sum_{t=1}^T \1\{\WL^i(\x_t) \neq y_t\}\1\{t \in U\} \leq (\half - \gamma)|U| + S.
	\end{equation}
	
	Now fix the internal randomness of $\WL^i$. Note that $\E_t[\1\{t \in U\}] = p_t^i = \frac{w_t^i}{\|\w^i\|_\infty}$, where $\E_t[\cdot]$ is the expectation conditioned on all the randomness of the booster until (and not including) round $t$. Define $\sigma = \sum_{t=1}^T p_t^i$.

	We now show using martingale concentration bounds that with high probability,
	\begin{equation}
		\sum_{t=1}^T\1\{\WL^i(\x_t) \neq y_t\} p_t^i \leq \sum_{t=1}^T \1\{\WL^i(\x_t) \neq y_t\}\1\{t \in U\} + \tilde{O}\left(\sqrt{\sigma}\right) \label{eq:martingale1}
	\end{equation}
	and
	\begin{equation} \label{eq:martingale2}
		|U| \leq \sigma + \tilde{O}\left(\sqrt{\sigma}\right).
	\end{equation}
	Here, the $\tilde{O}(\cdot)$ notation suppresses dependence on $\log\log(T)$.

	To prove inequality (\ref{eq:martingale1}), consider the martingale difference sequence
	\[X_t = \1\{\WL^i(\x_t) \neq y_t\}\1\{t \in U\} - \1\{\WL^i(\x_t) \neq y_t\} p_t^i.\]
	Note that $|X_t| \leq 1$, and the conditional variance satisfies
	\[\Var_t[X_t | X_1, X_2, \ldots, X_{t-1}] \leq p_t^i.\]
	Then, by Lemma~2 of \citet{BDHKRT}, for any $\delta < 1/e$ and assuming $T \geq 4$, with probability at least $1 - \log_2(T)\delta$, we have 
	\[\sum_{t=1}^T X_t \leq 2 \max\left\{2\sqrt{\sigma}, \sqrt{\ln(\deltainv)}\right\} \sqrt{\ln(\deltainv)} = \tilde{O}(\sqrt{\sigma}),\]
	by choosing $\delta \ll \frac{1}{\log_2(T)}$. This implies inequality (\ref{eq:martingale1}). Inequality (\ref{eq:martingale2}) is proved similarly. Note that these high probability bounds are conditioned on the internal randomness of $\WL^i$. By taking an expectation of this conditional probability over the internal randomness of $\WL^i$, we conclude that inequalities (\ref{eq:martingale1}) and (\ref{eq:martingale2}) hold with high probability unconditionally.

	Via a union bound, inequalities (\ref{eq:mistake-bound}), (\ref{eq:martingale1}) and (\ref{eq:martingale2}) all hold simultaneously with high probability, which implies that
	\begin{equation} \label{eq:martingale-combined}
		\sum_{t=1}^T\1\{\WL^i(\x_t) \neq y_t\} p_t^i \leq (\half - \gamma)\sigma + S + \tilde{O}\left(\sqrt{\sigma}\right).
	\end{equation}
	Using the facts that $p_t^i = \frac{w_t^i}{\|\w^i\|_\infty}$ and $\1\{\WL^i(\x_t) \neq y_t\} = \frac{1 - z_t^i}{2}$ and simplifying, we get
	\begin{align*}
		\w^i \cdot \z^i &\geq 2\gamma \|\w^i\|_1 - 2S\|\w^i\|_\infty - \tilde{O}(\sqrt{\|\w^i\|_1 \|\w^i\|_\infty})\\
		&\geq 2\gamma \|\w^i\|_1 - 2S\|\w^i\|_\infty - \gamma \|\w^i\|_1 - \tilde{O}(\tfrac{\|\w^i\|_\infty}{\gamma})\\
		&= \gamma \|\w^i\|_1 - 2S\|\w^i\|_\infty - \tilde{O}(\tfrac{\|\w^i\|_\infty}{\gamma}).
	\end{align*}
	The second inequality above follows from the arithmetic mean-geometric mean inequality. This gives us the desired bound. The high probability bound for all weak learners follows by taking a union bound.
\end{proof}

\section{Proof of Lemma \ref{lem:binom}}\label{app:lem:binom}
\begin{proof}
Let $X \sim B(m, p)$ be a binomial random variable 
where $m = N - i$ and $p = 1/2+\gamma/2$.
Also let $q = 1 - p$ and $F_X$ be the CDF of X. 
By the definition of $w_t^i$, we have $w_t^i \leq \frac{1}{2}\max_k \Pr\{X = k\} $.
We will approximate $X$ by a Gaussian random variable 
$G \sim N(mp, mpq)$ with density function $f$ and CDF $F_G$.
Note that 
\begin{align*}
|\Pr\{X=k\} - \int_{k-1}^{k} f(G) dG | =\;& |\left(F_X(k) - F_X(k-1)\right) - \left(F_G(k) - F_G(k-1)\right) | \\
\leq\;& | F_X(k) - F_G(k) | + | F_X(k-1) - F_G(k-1) |.
\end{align*}
So by applying Berry-Esseen theorem to the above two CDF differences between $X$ and $G$,
we arrive at 
$$ \left|\Pr\{X=k\} - \int_{k-1}^{k} f(G) dG\right| \leq \frac{2C(p^2+q^2)}{\sqrt{mpq}} ,$$
where $C$ is the universal constant stated in Berry-Esseen theorem.
It remains to point out that
\begin{align*}
\Pr\{X=k\} &\leq \int_{k-1}^{k} f(G) dG + \frac{2C(p^2+q^2)}{\sqrt{mpq}} \\
&\leq \max_{G\in R} f(G) + \frac{2C(p^2+q^2)}{\sqrt{mpq}} \\
&= \frac{1}{\sqrt{2\pi mpq}} + \frac{2C(p^2+q^2)}{\sqrt{mpq}} 
= O\left(\frac{1}{\sqrt{m}}\right),
\end{align*}
since $pq = 1/4-\gamma^2/4 \geq 3/16$.
\end{proof}

\end{document}